\documentclass[fleqn]{llncs}
 \pdfoutput=1 
\usepackage{url}
\usepackage{wrapfig}
\usepackage[all]{xy}
\usepackage{amsmath}
\usepackage{mathtools}
\usepackage{algorithm}
\usepackage{algpseudocode}
\usepackage{color}

\usepackage{array}
\usepackage{subfigure}
\newcommand{\bee}{\textsf{BEE}}

\newcommand{\SB}{\textsf{sb}}
\newcommand{\set}[1]{\left\{
      \begin{array}{l}#1\end{array}
      \right\}}
\newcommand{\sset}[2]{\left\{~#1  \left|
      \begin{array}{l}#2\end{array}
    \right.     \right\}}
\newcommand\tuple[1]{\langle #1 \rangle}

\newcommand{\GG}{{\cal G}}
\newcommand{\RR}{{\cal R}}
\newcommand{\etal}{{\it et al.\ }}
\newcommand\true{\mathit{true}}
\newcommand\false{\mathit{false}}

\title{Breaking Symmetries in Graph Search \\with Canonizing Sets}
\author{Avraham Itzhakov \and Michael Codish}
\institute{
  Department of Computer Science,
  Ben-Gurion University of the Negev, Israel}

\begin{document}

\maketitle

\begin{abstract}

  There are many complex combinatorial problems which involve
  searching for an undirected graph satisfying given constraints. 
  Such problems are often highly challenging because of the
  large number of isomorphic representations of their solutions.
  This paper introduces effective and compact, complete symmetry
  breaking constraints for small graph search.
  Enumerating with these symmetry breaks generates all and
  only non-isomorphic solutions.
  For small search problems, with up to $10$ vertices, we compute
  instance independent symmetry breaking constraints.
  For small search problems with a larger number of vertices we
  demonstrate the computation of instance dependent constraints which
  are complete.
  We illustrate the application of complete symmetry breaking
  constraints to extend two known sequences from the OEIS
  related to graph enumeration.
  We also demonstrate the application of a generalization of our
  approach to fully-interchangeable matrix search problems.
\end{abstract}

\section{Introduction}

Graph search problems are about the search for a graph which satisfies
a given set of constraints, or to determine that no such graph
exists. Often graph search problems are about the search for the
set of all graphs, modulo graph isomorphism, that satisfy the
given constraints. 
Graph search problems are typically invariant under graph
isomorphism. Namely, if $G$ is a solution, then any graph
obtained by permuting the vertices of $G$ is also a solution.
When seeking solutions, the size of the search space is significantly
reduced if symmetries are eliminated. The search space can be explored
more efficiently when avoiding paths that lead to symmetric solutions
and avoiding also those that lead to symmetric non-solutions.

One common approach to eliminate symmetries is to introduce symmetry
breaking
constraints~\cite{Puget93,Crawford96,Shlyakhter2007,Walsh06}
which rule out isomorphic solutions thus reducing the size of the
search space while preserving the set of solutions.
Ideally, a symmetry breaking constraint is satisfied by a single
member of each equivalence class of solutions, thus drastically
restricting the search space. However, computing such symmetry
breaking constraints is, most likely, intractable in
general~\cite{Crawford96}.
In practice, symmetry breaking constraints typically rule out some,
but not all of the symmetries in the search and, as noted in the
survey by Walsh~\cite{Walsh12}, often a few simple constraints rule
out most of the symmetries.

Shlyakhter~\cite{Shlyakhter2007} notes that the core difficulty is to
identify a symmetry-breaking predicate which is both \emph{effective}
(rules out a large portion of the search space) and \emph{compact} (so
that checking the additional constraints does not slow down the
search).
In~\cite{CodishMPS14}, Codish \etal\ introduce a symmetry breaking
constraint for graph search problems. Their constraint is compact,
with size polynomial in the number of graph vertices, and shown to be
effective but it does not eliminate all of the symmetries in the
search.

There is a large body of research that concerns identifying symmetries
in a given graph. In this setting, finding symmetries is about
detecting graph automorphisms. A typical application is in the context
of SAT solving as described for example 
in~\cite{AloulSM06,Aloul10,KatebiSM10,KatebiSM12,KatebiSM12b}.
In this paper the setting is  different as the graph is not given
but rather is the subject of the search problem.

In this paper we adopt the following terminology.  Symmetry breaking
constraints that break all of the symmetries, or more precisely, that
are satisfied by exactly one solution in each symmetry class, are
called \emph{complete}. Symmetry breaking constraints which are sound
i.e, satisfied by at least one solution in each symmetry class, but
not complete are called \emph{partial}. If a symmetry breaking
constraint is satisfied exactly by the canonical representatives of the
symmetry classes, it is called \emph{canonizing}. Note that canonizing
symmetry breaking constraints are also complete.

Computing all solutions to a graph search problem with partial
symmetry breaking constraints is a two step process. First one
generates the set $S$ of solutions to the constraints, and then one
applies a graph isomorphism tool, such as \texttt{nauty}~\cite{nauty}
to reduce $S$ modulo isomorphism. Often, the number of solutions in
the first step is very large and then this method may fail to generate
the initial set of solutions.

This paper presents a methodology to compute small sets of static canonizing
symmetry breaking constraints for ``small'' graph search problems.
Consider for example the search for a graph with $n=10$ vertices. The search space
consists of $2^{45}$ graphs, whereas, there are only 12\,005\,168 such
graphs modulo isomorphism (see sequence \texttt{A000088} of the
OEIS~\cite{oeis}). In theory, to break all symmetries one could
construct a symmetry breaking constraint that considers all
$10!=3{,}628{,}800$ permutations of the vertices.
We will show how to construct a compact canonizing symmetry
breaking constraint for graph search problems on 10 vertices using
only 7853 permutations. 

Our approach can be applied, in the terminology of~\cite{Aloul2011},
both in an ``instance independent'' fashion and ``instance
dependent''.  When ``instance independent'', it generates canonizing
symmetry breaking constraints for any graph search problem and in this
setting it applies to break all symmetries in graph search problems on
up to 10 vertices. When ``instance dependent'', it generates canonizing
symmetry breaking constraints which apply to break symmetries in
larger graphs which are solutions of a given graph search
problem. These symmetry breaking constraints are typically smaller and
easier to compute than the corresponding ``instance independent''
constraints.
%
We illustrate the application of complete symmetry breaking
constraints, both instance independent and instance dependent, to
extend two known sequences from the OEIS related to graph
enumeration. 

We also observe that the derived symmetry constraints are ``solver
independent''.  They can be applied in conjunction with any constraint
solver to restrict the search to canonical solutions  of a given
search problem. 

The rest of this paper is structured as follows.
Section~\ref{sec:motivating} provides a motivating example. 
Section~\ref{sec:prelim} presents preliminary definitions and
notation.
Section~\ref{sec:canonizing} describes how we compute complete and
canonizing symmetry breaking constraints. First, in
Section~\ref{subsec:indep}, for instance independent graph search
problems. Then, in Section~\ref{sec:dep_canonizing}, for a given graph
search problem.
Section~\ref{sec:matrixmodels} demonstrates a generalization of our
approach to matrix search problems and illustrates its impact when
solving  the Equi-distant Frequency
Permutation Array problem ({EFPA}).
Section~\ref{sec:conclusion} concludes.

\section{A Motivating Example}
\label{sec:motivating}

A classic example of a graph search problem relates to the search for
Ramsey graphs~\cite{Rad}. The graph $R(s,t;n)$ is a simple graph with
$n$ vertices, no clique of size $s$, and no independent set of size
$t$.  Figure~\ref{fig:r33graph} illustrates a $R(3,3;5)$ graph. The
graph contains no 3-clique and no 3-independent set. A Ramsey
$(s,t)$-graph is a $R(s,t;n)$ graph for some $n$. The set of all
$R(s,t;n)$ graphs, modulo graph isomorphism, is denoted $\RR(s,t;n)$.
Ramsey Theory tells us that there are only
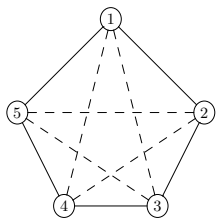
\begin{figure}  \centering
  \resizebox{.25\linewidth}{!}{
$$
\newcommand{\xyo}[1]{*+<0.8em>[o][F-]{#1}}
\xymatrix@!=1mm{
  &&\xyo{1}\ar@{-}[rrdd]\ar@{-}[lldd]\ar@{--}[rdddd]\ar@{--}[ldddd]&& \\
  \\
  \xyo{5}\ar@{-}[rdd]\ar@{--}[rrrr]\ar@{--}[rrrdd]&&&& 
  \xyo{2}\ar@{-}[ldd]\ar@{--}[llldd]\\
  \\
  &\xyo{4}\ar@{-}[rr]&&\xyo{3}&\\
}
$$
}
  \caption{A $R(3,3;5)$ Ramsey graph: edges denoted by solid lines and
    non-edges by dashed.}
  \label{fig:r33graph}
\end{figure}
 a finite number of Ramsey
$(s,t)$-graphs for each $s$ and $t$, but finding all such graphs, or
even determining the largest $n$ for which they exist, is a famously
difficult problem.
It is unknown, for example, if there exists a $R(5,5;43)$ graph and
the set $\RR(4,5;24)$ has yet to be been fully determined, although
350{,}904 non-isomorphic graphs are known to belong to $\RR(4,5;24)$.

Solving the graph search problem to find all $R(3,4;8)$ graphs without
any symmetry breaking constraint results in a set of 17{,}640
graphs. Then, applying \texttt{nauty}~\cite{nauty} to these solutions
identifies precisely 3 solutions modulo graph isomorphism.
Introducing a partial symmetry breaking constraint as described in
\cite{DBLP:conf/ijcai/CodishMPS13} in the search to enumerate all
$R(3,4;8)$ graphs computes only 11 graphs in a fraction of the time
required to compute the full set of solutions. These too can then be
reduced applying \texttt{nauty} to obtain the 3 canonical solutions.
Application of a complete symmetry breaking constraint as proposed in
this paper results in the exact set of 3 non-isomorphic solutions.

\section{Preliminaries}
\label{sec:prelim}
Throughout this paper we consider finite and simple graphs (undirected
with no  self loops). The set of simple graphs on $n$ nodes is
denoted $\GG_n$.  We assume that the vertex set of a graph,
$G = (V,E)$, is $V=\{1,\ldots,n\}$ and represent $G$ by its $n\times
n$ adjacency matrix $A$ defined by
\[A_{i,j}=    \begin{cases}1 & \mbox{if } (i,j) \in E\\
                                     0          & \mbox{otherwise}
                       \end{cases}
\]

An $n$-vertices graph search problem is a predicate $\varphi$ on an
$n\times n$ matrix $A$ of Boolean variables $A_{i,j}$; and a solution
to a graph search problem $\varphi$ is a satisfying assignment (to the
variables in $A$) of the conjunction $\varphi(A) \wedge adj^n(A)$ where
$adj^n(A)$ states that $A$ is an $n\times n$ adjacency matrix:
\begin{equation}
\label{constraint:simple}
adj^n(A) = \underbrace{\bigwedge_{1\leq i\leq n} 
                \hspace{-2mm}(\neg A_{i,i})}_{(a)}~~\wedge
          \underbrace{\bigwedge_{1\leq i<j\leq n} 
                \hspace{-4mm}(A_{i,j}\leftrightarrow A_{j,i})}_{(b)}
\end{equation}

In Constraint~(\ref{constraint:simple}), the left part (a) states that
there are no self loops and  the right part (b) states that the edges
are undirected.
The set of solutions of a graph search problem is denoted
$sol(\varphi)$ and when we wish to make the variables explicit we
write $sol(\varphi(A))$. The set $sol(\varphi)$ is typically viewed as
a set of graphs. Note that $sol(\true)=\GG_n$.
The following presents two examples of graph search problems which we
will refer to in rest of the paper.


\begin{example}\label{ex:ramsey}\em
  The Ramsey graph $R(s,t;n)$ is a simple graph with $n$ vertices, no
  clique of size $s$, and no independent set of size $t$. The set of
  all $R(s,t;n)$ graphs, modulo graph isomorphism, is denoted
  $\RR(s,t;n)$.
  The search for a Ramsey graph is a graph search problem where we
  take the following $\varphi_{R(s,t;n)}$ as the predicate $\varphi$. Here we denote by
  $\wp_s[n]$ (respectively $\wp_t[n]$) the set of subsets of size $s$
  (respectively $t$) of $\set{1,\ldots,n}$. The left conjunct (a)
  states that there is no clique of size $s$ in the graph, and the
  right conjunct (b) that there is no independent set of size $t$.
\begin{equation}
  \label{eq:ramsey}
  \varphi_{(s,t;n)}(A) = \hspace{-2mm} 
                \underbrace{\bigwedge_{I\in \wp_s[n]}\hspace{-1mm}
                \bigvee \sset{\neg A_{i,j}}{i,j \in I, \\i<j}}_{(a)} ~~\land
                \underbrace{\bigwedge_{I\in \wp_t[n]}\hspace{-1mm}
                \bigvee \sset{A_{i,j}}{i,j \in I, \\i<j} }_{(b)}
\end{equation}
\end{example}

\begin{example}\label{ex:claw}\em
  A graph is claw-free if it does not contain the complete bipartite
  graph $K_{1,3}$ (sometimes called a ``claw'') as a subgraph. 
  The claw free graph search problem is formalized by taking the
  following $\varphi_{\mathit{cf}(n)}$ as the predicate $\varphi$. Each clause in the
  conjunction expresses for $i,j,k,\ell$ that there is no subgraph
  $K_{1,3}$ between $\{i\}$ and $\{j,k,\ell\}$.
\begin{equation}
  \label{eq:claw}
  \varphi_{\mathit{cf}(n)}(A) = 
    \bigwedge \sset{      \begin{array}{l}
                            \neg A_{i,j} \lor \neg A_{i,k} \lor \neg A_{i,\ell}  \\ 
                             \lor A_{j,k} \lor A_{k,\ell} \lor \lor A_{j,\ell}
                          \end{array}
                   }
                   {1\leq i\leq n,~ i\neq j,~i\neq k,\\
                    i\neq\ell,~1\leq j<k<\ell\leq n} 
\end{equation}
\end{example}
\medskip

The set of permutations $\pi:\{1,\ldots,n\}\rightarrow\{1,\ldots,n\}$
is denoted $S_n$. For $G = (V,E) \in \mathcal{G}_n$ and $\pi \in S_n$,
we define $\pi(G)=\{V, \{(\pi(u),\pi(v))| (u,v)\in E)\}$.
Permutations act on adjacency matrices in the natural way: If $A$ is
the adjacency matrix of a graph $G$, then $\pi(A)$ is the adjacency
matrix of $\pi(G)$ obtained by simultaneously permuting with $\pi$
the rows and columns of $A$.  
We adopt the tuple notation $[\pi(1),\ldots,\pi(n)]$ for a permutation
$\pi:\{1, \ldots, n\} \rightarrow \{1, \ldots, n\}$.

Two graphs $G_1,G_2\in \mathcal{G}_n$ are isomorphic, denoted
$G_1\approx G_2$, if there exists a permutation $\pi\in S_n$ such that
$G_1 = \pi(G_2)$. Sometimes we write $G_1\approx_\pi G_2$ to emphasize
that $\pi$ is the permutation such that $G_1=\pi(G_2)$.
For sets of graphs $H_1, H_2$, we say that $H_1\approx H_2$ if for
every $G_1\in H_1$ (likewise in $H_2$) there exists $G_2\in H_2$
(likewise in $H_1$) such that $G_1\approx G_2$.

We consider an ordering on graphs, defined viewing their adjacency
matrices as strings. Because adjacency matrices are symmetric with
zeroes on the diagonal, it suffices to focus on the upper triangle
parts of the matrices \cite{Cameron1985}.
\begin{definition}[ordering graphs]\label{def:order}
  Let $G_1,G_2 \in \mathcal{G}_n$ and let $s_1, s_2$ be the strings
  obtained by concatenating the rows of the upper triangular parts of
  their corresponding adjacency matrices $A_1, A_2$ respectively.
  Then, $G_1 \preceq G_2$ if and only if $s_1\preceq_{lex} s_2$. We
  also write  $A_1 \preceq A_2$.
\end{definition}

One way to define the canonical representation of a graph is to take 
the smallest graph in the lexicographic order (i.e \textsc{LexLeader}) in each
equivalence class of isomorphic graphs \cite{Read78}. In this paper
we follow this definition for canonicity.  

\begin{definition}[canonicity]\label{def:canonizing}
  Let $G\in\mathcal{G}_n$ be a graph, $\Pi\subseteq S_n$, and denote
  the predicate $\min_\Pi(G) = \bigwedge\sset{G \preceq
  \pi(G)}{\pi\in \Pi}$. 
  We say that $G$ is canonical if $\min_{S_n}(G)$. We say
  that $\Pi$ is canonizing if $\forall_{G\in
    \mathcal{G}_n}.\min_\Pi(G)\leftrightarrow \min_{S_n}(G)$.
\end{definition}

Observe that in Definitions~\ref{def:order} and~\ref{def:canonizing}, the
order is defined on given graphs.  Often, we consider the same
relation, but between adjacency matrices that contain propositional
variables (representing unknown graphs, as in the case for graph
search problems). Then, the expressions $A_1\preceq A_2$ and
$\min_\Pi(A)$ are viewed as a Boolean constraints on the variables in
the corresponding matrices.

\newcolumntype{C}[1]{@{}>{\rule[0.5\dimexpr-#1+1.2ex]{0pt}{#1}\hfil$}p{#1}<{$\hfil}@{}}

\begin{figure}
  \centering
  \subfigure[Graph $G$]{
        \centering
        $\left[
         \begin{array}{C{5mm}C{5mm}C{5mm}C{5mm}}
             0 & 0 & 0 & 1 \\
             0 & 0 & 1 & 0 \\
             0 & 1 & 0 & 1 \\
             1 & 0 & 1 & 0
          \end{array}\right]$
        \label{gANDp1}}
\quad
  \subfigure[Graph $\pi_1(G)$]{
        \centering
        $\left[
         \begin{array}{C{5mm}C{5mm}C{5mm}C{5mm}}
             0 & 0 & 1 & 0 \\
             0 & 0 & 0 & 1 \\
             1 & 0 & 0 & 1 \\
             0 & 1 & 1 & 0
          \end{array}\right]$
        \label{gANDp2}}
\quad
  \subfigure[Graph $\pi_2(G)$]{
        \centering
        $\left[
         \begin{array}{C{5mm}C{5mm}C{5mm}C{5mm}}
             0 & 0 & 0 & 1 \\
             0 & 0 & 1 & 1 \\
             0 & 1 & 0 & 0 \\
             1 & 1 & 0 & 0
          \end{array}\right]$
        \label{gANDp3}}
\quad
  \subfigure[Graph $\pi_3(G)$]{
        \centering
        $\left[
         \begin{array}{C{5mm}C{5mm}C{5mm}C{5mm}}
             0 & 0 & 1 & 0 \\
             0 & 0 & 0 & 1 \\
             1 & 0 & 0 & 1 \\
             0 & 1 & 1 & 0
          \end{array}\right]$
        \label{gANDp4}}

      \caption{A graph and its isomorphic representations according to: $\pi_1=[2,1,3,4]$,
        $\pi_2=[1,3,2,4]$, and $\pi_3=[1,2,4,3]$.}
  \label{fig:gANDp}
\end{figure}

\begin{example}\label{ex:canSet}\em
  It turns out that $\Pi= \set{[2,1,3,4], [1,3,2,4], [1,2,4,3]}$ is
  canonizing for $\GG_4$. Namely, with only three permutations we
  express the information present in all $4!=24$ elements of $S_4$.
  So for instance, the graph $G$ depicted in Figure~\ref{gANDp1} is
  canonical because it is smaller than its three permutations with
  respect to $\Pi$ detailed as Figures~\ref{gANDp2}, \ref{gANDp3},
  and~\ref{gANDp4}.  We come back to elaborate on why $\Pi$ is
  canonizing in Example~\ref{ex:minPi}.
\end{example}

\begin{definition}[symmetry break]
  Let $\varphi(A)$ be a $n$-vertices graph search problem and
  $\sigma(A)$ a propositional formula on the variables in $A$. We say
  that $\sigma$ is a symmetry break for $\varphi$ if $sol(\varphi(A))
  \approx sol(\varphi(A)\wedge\sigma(A))$.  If the graphs of
  $sol(\varphi(A)\wedge\sigma(A))$ are mutually non-isomorphic then we
  say that $\sigma$ is complete. Otherwise we say that $\sigma$ is
  partial.  If the graphs of $sol(\varphi(A)\wedge\sigma(A))$ are
  canonical then we say that $\sigma$ is canonizing.
\end{definition}

\begin{lemma}\label{lemma:sb}
  Let $\Pi$ be a canonizing set of permutations for graphs of size
  $n$. Then $\min_\Pi$ is a canonizing symmetry break for
  any graph search problem on $n$ vertices.
\end{lemma}
\begin{proof}
  Let $A$ be a solution to a graph search problem on $n$ vertices and let $\Pi$
  be a canonizing set for graphs with $n$ vertices. In order to prove that $min_{\Pi}$ is
  a canonizing symmetry break it is sufficient to show that only the canonical member in $iso(A)=
  \{ \pi(A) | \forall \pi \in S_n \}$ satisfies $min_{\Pi}$. $\Pi$ is a canonizing
  set thus by definition $\forall G\in \mathcal{G}_{n} : min_{\Pi}(G) \leftrightarrow min_{S_{n}}(G)$. 
  Since only the canonical graph in $iso(A)$ satisfies $min_{S_{n}}$ it follows
  that it is the only one which satisfies $min_{\Pi}$.

\end{proof}

\begin{corollary}
  $\min_{S_n}(A)$ is a canonizing symmetry break for
  any graph search problem on $n$ vertices.
\end{corollary}

\begin{example}\label{ex:minPi}\em 
Consider the canonizing set $\Pi$ from Example~\ref{ex:canSet}
and the following $4\times 4$ adjacency matrix A:\\
\begin{tabular}{ll}
\hspace{2cm}$ A=
\left[
  \begin{array}{C{5mm}C{5mm}C{5mm}C{5mm}}
     0 & a & b & c \\
     a & 0 & d & e \\
     b & d & 0 & f \\
     c & e & f & 0
  \end{array}\right]
$
   &\hspace{2cm}
$\Pi= \set{ ~[2,1,3,4],\\ ~ [1,3,2,4],\\ ~ [1,2,4,3]}$\\
\end{tabular}\\
Then, by Definition~\ref{def:canonizing} and Lemma~\ref{lemma:sb},
\[\mbox{$\min_\Pi(A)$}= (\mathit{a b  c d e f} \preceq_{lex} \mathit{a d e b c f})~\wedge~
              (\mathit{a b c d e f} \preceq_{lex}  \mathit{b a c d f e})~\wedge~
              (\mathit{a b c d e f} \preceq_{lex}  \mathit{a c b d e f})
\]
and this simplifies using properties of lexicographic
orderings to:
\[\mbox{$\min_\Pi(A)$}= (b c \preceq_{lex} d e) ~\wedge~
                       (a e \preceq_{lex} \mathit{b f}) ~\wedge~
                       (b d \preceq_{lex} c e)
\]
To verify that $\Pi$ is indeed canonizing one should consider each of
the permutations in $\pi\in S_4\setminus\Pi$ and prove that
$\min_\Pi(A)\Rightarrow A\preceq \pi(A)$ where $A$ is the
variable matrix detailed above.
For example, when $\pi=[2,1,4,3]$, $A\preceq \pi(A)$
means $abcdef \preceq_{lex} aedcbf$ which simplifies to
$bc\preceq_{lex}ed$ and we need to show that $(b c \preceq_{lex} d e)
~\wedge~ (a e \preceq_{lex} \mathit{b f}) ~\wedge~ (b d \preceq_{lex}
c e) \Rightarrow bc\preceq_{lex}ed$. This is not difficult to check.

\end{example}

Clearly, for any set of permutations $\Pi\subset S_n$ the predicate
$\min_\Pi$ is a partial symmetry break for graph search problems.
In~\cite{DBLP:conf/ijcai/CodishMPS13}, Codish \etal introduce the
following symmetry break for graph search problems where $A_i$ denotes
the $i^{th}$ row of the adjacency matrix $A$ and $\preceq_{\{i,j\}}$
denotes the lexicographic comparison on strings after removing their
$i^{th}$ and $j^{th}$ elements.
\begin{definition}[lexicographic symmetry break, \cite{DBLP:conf/ijcai/CodishMPS13}]
\label{def:SBlexStar}
  Let $A$ be an $n\times n$ adjacency matrix. Then,
  \[\SB^*_\ell(A) = \bigwedge_{1\leq i<j\leq n}  A[i]\preceq_{\{i,j\}} A[j]\]
\end{definition}
It is not difficult to observe that $\SB^*_\ell$ is equivalent to the
predicate $\min_\Pi$ where $\Pi$ is the set of permutations that swap
a single pair $(i,j)$ with $1\leq i<j\leq n$. 

\paragraph{\bf The experimental setting.~} In this paper all
computations are performed using the Glucose 4.0 SAT
solver~\cite{Glucose}. Encodings to CNF are obtained using the finite-domain
constraint compiler \bee~\cite{jair2013}. \bee\ facilitates
applications to find a single (first) solution, or to find all
solutions for a constraint, modulo a specified set of variables.
When solving for all solutions, \bee\ iterates with the
SAT solver, adding so called \emph{blocking clauses} each time another
solution is found. This technique, originally due to
McMillan~\cite{McMillan2002}, is simplistic but suffices for our
purposes.
All computations were performed on a cluster with a total of $228$
Intel E8400 cores clocked at 2 GHz each, able to run a total of $456$
parallel threads. Each of the cores in the cluster has computational
power comparable to a core on a standard desktop computer.  Each SAT
instance is run on a single thread.

\section{Canonizing Symmetry Breaks}
\label{sec:canonizing}

The observation made in Example~\ref{ex:canSet}: that a canonizing set
for graphs with $n$ vertices can be much smaller than $n!$, motivates
us to seek ``small'' canonizing sets that might be applied to
introduce canonizing symmetry breaking constraints for graph search
problems.
First, we describe the application of this approach to compute  
relatively small \emph{instance independent} canonizing sets, which
induce general purpose symmetry breaks that can be used for any graph
search problem. We compute these sets for graphs with $n \leq 10$ vertices.
We illustrate their application when breaking all symmetries in the search
for Ramsey and claw-free graphs. 

Second, we apply our methods to compute \emph{instance dependent}
canonizing sets which are computed for a given graph search
problem. Namely, these sets promise that only non-isomorphic solutions
will be generated when enumerating all solutions for the given graph
search problem that satisfy its corresponding canonizing symmetry
breaks. However, these sets are not necessarily canonizing for other
graph search problem. We show that such canonizing sets can be
computed for larger graphs (compare to instance independent canonizing
sets) and their usage is illustrated to enumerate all non-isomorphic
highly irregular graphs up to $20$ vertices.

\subsection*{Computing Canonizing Sets}

To compute a canonizing set of permutations for graph search problem $\varphi$
on $n$ vertices we start with some initial set of permutations $\Pi$
(for simplicity, assume that $\Pi=\emptyset$). Then, incrementally
apply the step specified in lines \ref{line:while}--\ref{line:bla} of
Algorithm~\ref{CanSetAlgorithm}, as long as the stated condition
holds.
\begin{algorithm}
\begin{algorithmic}[1]
\Procedure{Compute-Canonizing-Set($\Pi$, $\varphi$)}{}
\While{ $\exists G\in sol(\varphi)$ $\exists \pi\in S_n$ such that
        $\min_\Pi(G)$  and $\pi(G)\prec G$}\label{line:while}
\State  $\Pi=\Pi\cup\{\pi\}$\label{line:bla}
\EndWhile
\State \textbf{return} $\Pi$
\EndProcedure
\end{algorithmic}
\caption{Compute Canonizing Set}
\label{CanSetAlgorithm}
\end{algorithm} 

\begin{lemma}
  Algorithm~\ref{CanSetAlgorithm} terminates and returns a canonizing
  set $\Pi$ for the graph search problem $\varphi$.
\end{lemma}
 
\begin{proof}
  Each step in the algorithm adds a permutation (at Line
  \ref{line:bla}) and the number of permutations is bound.
  When the algorithm terminates with $\Pi$ then 
  for $G\in sol(\varphi)$, if $\min_\Pi(G)$ holds then there is no $\pi\in
  S_n$ such that $\pi(G)\prec G$. So, $G \preceq \pi(G)$ for all $\pi
  \in S_n$ and therefore $\min_{S_n}(G)$ holds.
\end{proof}

\begin{figure}\small
  \centering
    \begin{eqnarray}
      {perm}^{n}(\pi) &=& 
                  \bigwedge_{1\leq i\leq n} \hspace{-1mm} 1\leq \pi_{i}\leq n 
                  ~~\land~~ 
                  \mathtt{allDifferent}(\pi)
     \label{constraint:permutation}
\\
      {iso}^{n}(A,B,\pi) &=& \hspace{-2mm}
        \bigwedge_{1\leq i,j\leq n}
        \left(\begin{array}{l}
           B_{i,j} \Leftrightarrow ~
               \bigvee_{1\leq i',j'\leq n}
               \left(\pi_{i'} = i \land \pi_{j'} = j  \land A_{i',j'}\right)
        \end{array}\right)
     \label{constraint:iso}
\\[1ex]
     {\mathit{alg_1^n}}(\Pi,\varphi) &=& adj^{n}(A) ~\land~ 
                  adj^{n}(B) ~\land~
                   {perm}^{n}(\pi)~\land~  {iso}^{n}(A,B,\pi)~\land
\label{constraint:step} \\
     & & \wedge~\mbox{$\min_{\Pi}(A)$}   ~\land~ A \succ B ~\land~ \varphi(A)
     \nonumber
\end{eqnarray}

\caption{Constraints for Algorithm~\ref{CanSetAlgorithm} where $A$ and
  $B$ are $n\times n$ Boolean matrices and
  $\pi=\tuple{\pi_1,\ldots,\pi_n}$ is a vector of integer variables
  (with domain $\{1,\ldots,n\}$).}
  \label{fig:cc}
\end{figure}

Drawing on the discussion in~\cite{Luks2004,Babai1983,Crawford96} we
do not expect to find a polynomial time algorithm to compute a
canonical (or any other complete) symmetry breaking constraint for
graph search problems based on Definition~\ref{def:canonizing}. Thus
it is also unlikely to find an efficient implementation of
Algorithm~\ref{CanSetAlgorithm}.
Our implementation of Algorithm~\ref{CanSetAlgorithm} is based on a
SAT encoding. We repeatedly apply a SAT solver to find a counter
example permutation which shows that $\Pi$ is not a canonizing set yet
and add it to $\Pi$, until an UNSAT result is obtained.
In the implementation of the algorithm, care is taken to use a single
invocation of the SAT solver so that the iterated calls to the solver
are incremental. The constraint model used is depicted as Figure~\ref{fig:cc}
%
where $A,B$ denote $n\times n$ matrices of propositional variables and
$\pi$ denotes a length $n$ vector of integer variables. 
Constraint~\ref{constraint:permutation} specifies that the parameter
$\pi$ is a permutation on $\set{1,\ldots,n}$. Each element of the vector
is a value $1\leq \pi_i\leq n$ and the elements are all different.
Constraint~\ref{constraint:iso} specifies that the parameters
$A,B$ represent isomorphic graphs via the parameter $\pi$.
Constraint~\ref{constraint:step} specifies the condition of the while
loop (line~\ref{line:while}) of Algorithm~\ref{CanSetAlgorithm}: $A$
is restricted to be a solution to the given graph search problem $\varphi$, $A$
and $B$ are constrained $B = \pi(A)$ to be isomorphic adjacency
matrices (see Constraint~(\ref{constraint:simple})) via the
permutation $\pi$. The constraint $\min_\Pi(A)$ is imposed and also
$A\succ B$. If ${\mathit{alg_1^n}}(\Pi,\varphi)$ is satisfiable, then
the permutation $\pi$ is determined by the satisfying assignment and
added to $\Pi$ as specified in (line~\ref{line:bla}) of
Algorithm~\ref{CanSetAlgorithm}.

We say that a canonizing set $\Pi$ of permutations is redundant if for
some $\pi\in\Pi$ the set $\Pi\setminus\{\pi\}$ is also canonizing.
Algorithm~\ref{CanSetAlgorithm} may compute a redundant set. For
example, if a permutation added at some point becomes redundant
in view of permutations added later. Algorithm~\ref{ReduceAlgorithm}
iterates on the elements of a canonizing set to remove redundant
permutations.

\begin{algorithm}[H]
\caption{Reduce method}
\label{ReduceAlgorithm}
\begin{algorithmic}[1]
\Procedure{Reduce}{$\Pi$,$\varphi$}
\For{ each $\pi \in \Pi$ } \label{reduce:select}
\If  {$\forall G \in sol(\varphi)$: 
        $\min_{\Pi\setminus \{ \pi \}}(G)$ $\Rightarrow$ $G \preceq \pi(G)$}
\label{reduce:test}
\State $\Pi= \Pi \setminus \{\pi\}$
\EndIf
\EndFor
\State return $\Pi$
\EndProcedure
\end{algorithmic}
\end{algorithm} 

\begin{lemma}
  If $\Pi$ is a canonizing set for the graph search problem $\varphi$,
  then so is $Reduce(\Pi,\varphi)$ computed by
  Algorithm~\ref{ReduceAlgorithm}.
\end{lemma}

\begin{proof}
Let $\Pi_{i}$ be the set obtained after considering the $i^{th}$ permutation
in Line~\ref{reduce:select} of Algorithm~\ref{ReduceAlgorithm}. The
initial set $\Pi_0$ is the input to the algorithm.
We prove that $\min_{\Pi_{i}} \leftrightarrow \min_{\Pi_{i+1}}$ 
and conclude that $\min_{\Pi} \leftrightarrow \min_{Reduce(\Pi)}$.
If no permutation was removed in step $i$ then $\Pi_{i+1} = \Pi_{i}$
and trivially $\min_{\Pi_{i}} \leftrightarrow \min_{\Pi_{i+1}}$.
Otherwise $\Pi_{i+1} = \Pi_{i} \setminus \{ \pi \}$ for a permutation
$\pi$ which satisfies $\forall G \in sol(\varphi)$:
$\min_{\Pi_{i+1}}(G)$ $\Rightarrow$ $G \preceq \pi(G)$.  Thus $\pi$ is
implied by the permutations in $\Pi_{i}$ and can be removed. Therefore
$\min_{\Pi_{i+1}}(G) \leftrightarrow \min_{\Pi_{i}}(G)$.
\end{proof}
Our implementation of Algorithm~\ref{ReduceAlgorithm} is based on a SAT
encoding.  The key is in the encoding for the test in
Line~\ref{reduce:test}.  Here, for the given $\Pi$ and $\pi\in\Pi$, we
encode the constraint
\begin{equation}\label{constraint:reduce_step}
 {alg_2}^n(\Pi,\varphi) = \underbrace{adj^{n}(A)}_{(a)}
   ~\land~ 
   \underbrace{\varphi(A) ~\land~ min_{\Pi \setminus \{ \pi \}}(A) ~\land~ \pi(A) \succ A}_{(b)}
\end{equation}
where the left part (a) specifies that $A$ is the $n\times n$
adjacency matrix of some graph (see
Constraint~(\ref{constraint:simple})), and the right part (b) is the
negation of the condition in Line~\ref{reduce:test}.
If this constraint is unsatisfiable then $\pi$ is redundant and
removed from~$\Pi$.  

\subsection{Instance Independent Symmetry breaks}
\label{subsec:indep}

Observe that if $\varphi = \true$ then $sol(\varphi)=\GG_n$. Applying
Algorithm~\ref{CanSetAlgorithm} to compute using
\texttt{{Compute-Canonizing-Set($\Pi$, $\true$)}} generates canonizing
symmetry breaks which apply for any graph search problem on n vertices
(i.e instance independent). This is true for any set of permutations
$\Pi$ but for simplicity assume $\Pi=\emptyset$.
%
%

Table~\ref{canSetTable} describes the computation of irredundant
instance independent canonizing permutation sets for $n\leq 10$ by
application of Algorithms~\ref{CanSetAlgorithm}
and~\ref{ReduceAlgorithm}.
The corresponding permutation sets can be obtained from
\url{http://www.cs.bgu.ac.il/~mcodish/Papers/Tools/canonizingSets}.
The first 3 columns indicate the number of graph vertices, $n$,
the number of permutations on $n$, and the number of non-isomorphic
graphs on $n$ vertices as specified by sequence \texttt{A000088} of
the OEIS~\cite{oeis}.
The forth and fifth columns indicate the size of the canonical set of
permutations computed using Algorithm~\ref{CanSetAlgorithm} and the
time to perform this computation.
Columns six and seven are the size of the reduced canonical set of
permutations after application of Algorithm~\ref{ReduceAlgorithm} and
the corresponding computation time.
Column seven is set in boldface. These numbers present the relatively
small size of the computed canonizing sets in comparison to the value
of $n!$. Using the symmetry breaks derived from these sets we have
generated the sets of all non-isomorphic graphs with up to 10 vertices
and verified that their numbers correspond to those in column three.
These are computed by solving the conjunction of
Constraint~(\ref{constraint:simple}) with the corresponding symmetry
breaking predicate $\min_\Pi$ the computation of which is described in
Table~\ref{canSetTable}.

The numbers in Table~\ref{canSetTable} also indicate the limitation of
complete symmetry breaks which apply to all graphs. We do not expect to succeed
to compute a canonizing set of permutations for $n=11$ and
even if we did succeed, we expect that the number of constraints that would
then need be added in applications would be too large to be effective.

\begin{table}[h]
\centering
\caption{Computing irredundant canonizing sets of permutations for $n \leq 10$.}
\label{canSetTable}
\begin{tabular}{|r|r|r|r|r|r|>{\bfseries}r|}
\hline 
\multicolumn{1}{|c|}{}  &  \multicolumn{1}{c|}{}  &   \multicolumn{1}{c|}{} &
\multicolumn{2}{c|}{Algorithm 1}  & \multicolumn{2}{c|}{Algorithm 2} \\ \hline
$n$ & $n!$ & can. graphs & \multicolumn{1}{c|}{time (sec.)} & \multicolumn{1}{c|}{can. set} & \multicolumn{1}{c|}{time (sec.)} & \multicolumn{1}{c|}{red. set} \\ \hline
 3  &           6 &            4 & 0.02      &   3        &   0.01   &    2   \\ 
 4  &          24 &           11 & 0.02      &   7        &   0.01   &    3   \\ 
 5  &         120 &           34 & 0.05      &  27        &   0.05   &    7   \\ 
 6  &         720 &          156 & 0.35      &  79        &   1.27   &   13   \\ 
 7  &      5\,040 &       1\,044 & 1.92      & 223        &  11.27   &   37   \\ 
 8  &     40\,320 &      12\,346 & 27.61     & 713        &  317.76  &  135   \\ 
 9  &    362\,880 &     274\,668 & 1\,108.13 & 4\,125     &  7\,623.20 &  842   \\ 
10  & 3\,628\,800 & 12\,005\,168 &  9.82 hr. & 20\,730    &  84 hr.  & 7\,853   \\ \hline
\end{tabular}
\end{table}

\subsection*{Computing Ramsey Graphs with Canonizing Symmetry Breaks}

Recall Example~\ref{ex:ramsey} where we introduce the  graph search
problem for Ramsey graphs.
Table~\ref{ramseyUpTo10} describes the computation of all $R(4,4;n)$
graphs for $n\leq 10$ using a SAT solver. The table compares two
configurations: one using the partial symmetry breaking predicate
$\SB^*_\ell$ defined in \cite{DBLP:conf/ijcai/CodishMPS13} and the
other using a canonizing symmetry break $\min_\Pi$ where $\Pi$ is the
canonizing set of permutations, the computation of which is described
in Table~\ref{canSetTable}.
For each configuration we detail the size of the SAT encoding (clauses
and variables), the time in seconds (except where indicated in hours)
to find all solutions using a SAT solver, and the number of solutions
found. Observe that the encodings using the canonizing symmetry breaks
are much larger. However the sat solving time is much smaller. For
$n=10$ the configuration with $\SB^*_\ell$ requires more than 33 hours
where as the configuration using $\min_\Pi$ completes in under 7
hours. 
Finally note that the computation with $\min_\Pi$ computes the
precise number of solutions modulo graph isomorphism as detailed for
example in~\cite{MR95}. These are the numbers in the rightmost column
set in boldface. The solutions computed using $\SB^*_\ell$ contain
many isomorphic solutions which need to be subsequently removed using
\texttt{nauty} or a similar tool. 
One might argue that the real cost in applying the complete symmetry
breaks should include their computation. To this end we note that
these are general symmetry breaking predicates applicable to any graph
search problem, and they are precomputed once.

\begin{table}
\centering
\caption{Computing $| \RR(4,4;n) |$ with canonizing symmetry breaking and $\SB^*_\ell$.}
\label{ramseyUpTo10}
\begin{tabular}{|@{~}r@{~}|r@{~}|r@{~}|r@{~}|r@{~}|r@{~}|r@{~}|r@{~}|>{\bfseries}r@{~}|}
\hline
  & 
\multicolumn{4}{c|}{partial sym. break $\SB^*_\ell$}  
                                         &
                                         \multicolumn{4}{c|}{canonizing
                                         sym. break} \\ \hline
             $n$ & \multicolumn{1}{c|}{clauses}
              & \multicolumn{1}{c|}{vars} 
              & \multicolumn{1}{c|}{sat (sec.)}
              & \multicolumn{1}{c|}{sols} 
              & \multicolumn{1}{c|}{clauses}
              & \multicolumn{1}{c|}{vars} 
              & \multicolumn{1}{c|}{sat (sec.)}
              & \multicolumn{1}{c|}{sols} \\ \hline
$4$  & 22   & 10  & 0.01    &  9        &  17        &  5        & 0.00   & 9      \\  
$5$  & 80   & 24  & 0.01    & 33        & 235        & 55        & 0.01   & 24          \\ 
$6$  & 195  & 48  & 0.02    & 178       & 315        & 72        & 0.01   & 84          \\ 
$7$  & 390  & 85  & 0.12    & 1\,478      & 1\,395       & 286       & 0.05   & 362         \\ 
$8$  & 690  & 138 & 4.91   & 16\,919  & 10\,885   & 2\,177      & 1.69    & 2\,079        \\ 
$9$  & 1\,122 & 210 & 745.72 & 227\,648 & 89\,877   & 17\,961  & 144.4   & 14\,701    \\ 
$10$ &1\,715 &304 &33.65 {hr.}& 2\,891\,024 & 1\,406\,100 & 281\,181 & 6.56 {hr.} &103\,706   \\ \hline

\end{tabular}
\end{table}


\subsection*{Computing Claw-Free Graphs with Canonizing Symmetry Breaks}

Recall Example~\ref{ex:claw} where we introduce the  graph search
problem for claw-free graphs.
The number of claw-free graphs for $n \leq 9$ vertices is detailed as
sequence \texttt{A086991} on the
OEIS~\cite{oeis}. Table~\ref{clawUpTo10} describes the search for claw
free graphs as a graph search problem. Then we use a SAT solver to
compute the set of all claw free graphs on $n$ vertices. We illustrate
that using canonizing symmetry breaks, and the results detailed in
Table~\ref{canSetTable}, we can compute the set of all claw free
graphs modulo graph isomorphism for $n\leq 10$ thus computing a new
value for sequence \texttt{A086991}. We comment that after this value
for $n=10$ was added to the OEIS, Falk H{\"{u}}ffner added further
values for $10<n\leq 15$. 
The column descriptions are the same as those for
Table~\ref{ramseyUpTo10}. 
For each configuration we detail the size of the SAT encoding (clauses
and variables), the time in seconds to find all solutions using a SAT
solver, and the number of solutions found. For this example the
computation with a complete symmetry break is more costly, but it does
return the precise set of canonical graphs.
The sequence in the right column are set in boldface. For $n\leq 9$,
these are the numbers of claw-free graphs as detailed in sequence
\texttt{A086991} of the OEIS~\cite{oeis}.
It is no coincidence that the number of variables indicated in the
columns of Tables~\ref{ramseyUpTo10} and~\ref{clawUpTo10} are almost
identical. These pertain to the variables in the adjacency graph and
those introduced to express the instance independent symmetry breaks.

\begin{table}
\centering
\caption{Computing claw-free graphs with canonizing symmetry breaking
  and $\SB^*_\ell$.} 
\label{clawUpTo10}
\begin{tabular}{|@{~}r@{~}|r@{~}|r@{~}|r@{~}|r@{~}|r@{~}|r@{~}|r@{~}|>{\bfseries}r@{~}|}
\hline
  & \multicolumn{4}{c|}{partial sym. break $\SB^*_\ell$}  & \multicolumn{4}{c|}{canonizing sym. break} \\ \hline 
              $n$ & \multicolumn{1}{c|}{clauses}
              & \multicolumn{1}{c|}{vars} 
              & \multicolumn{1}{c|}{sat (sec.)}
              & \multicolumn{1}{c|}{sols} 
              & \multicolumn{1}{c|}{clauses}
              & \multicolumn{1}{c|}{vars} 
              & \multicolumn{1}{c|}{sat (sec.)}
              & \multicolumn{1}{c|}{sols} \\ \hline
4   & 24    & 10   & 0.01    & 10        & 19       & 9        & 0.01  & 10       \\ 
5   & 90    & 24   & 0.01    & 32        & 245      & 55       & 0.01  & 26       \\ 
6   & 225   & 48   & 0.01    & 143       & 345      & 72       & 0.01  & 85       \\ 
7   & 460   & 85   & 0.04    & 819       & 1\,465     & 286      & 0.03      & 302      \\ 
8   & 830   & 138  & 0.86    & 5\,559      & 11\,025    & 2\,177     & 1.08      & 1\,285     \\ 
9   & 1\,374  & 210  & 28.72   & 42\,570     & 90\,129    & 17\,961    & 75.23     & 6\,170     \\ 
10  & 2\,135  & 304  & 2\,352.37 & 368\,998    & 1\,406\,520  & 281\,181   & 8797.23   & 34\,294 \\ \hline

\end{tabular}

\end{table}


\subsection{Instance Dependent Canonizing Symmetry Breaks}
\label{sec:dep_canonizing}

A canonizing set for a specific graph search problem $\varphi$ is
typically much smaller than a general canonizing set as the
constraints in $\varphi$ restrict the solution structure and hence
also the symmetries within the solution space.  We call such a set
\emph{instance dependent}.
In practice we can often compute instance dependent canonizing sets
for larger graphs with $n > 10$ vertices. 
For a given graph search problem $\varphi$, let us denote by
$\Pi_{\varphi}$ the canonizing set of permutations obtained from
\texttt{{Compute-Canonizing-Set($\emptyset$, $\varphi$)}} of
Algorithm~\ref{CanSetAlgorithm}.
Solutions of $\varphi$ obtained with the induced symmetry break
predicate $\min_{\Pi_\varphi}$ are guaranteed to be pairwise
non-isomorphic.

In this section we demonstrate the application of instance dependent canonizing sets.
Here we consider a search problem for which we seek a graph that has
a particular given degree sequence.
%
%

%
%
%
%
%
A degree sequence is a monotonic non-increasing sequence of the vertex
degrees of a graph. Degree sequences are a natural way to break a graph
search problem into independent cases (one for each possible degree
sequence). Thus the search for a solution or all solutions can be done
in parallel.

Since a degree sequence induces a partition on the vertex set, in
order to compute an instance dependent canonizing symmetry break with
respect to a degree sequence, a constraint stating that $B$ has the
same degree sequence as $A$ needs to be added to
(\ref{constraint:step}$^\prime$). The following specifies that an
adjacency matrix complies to a given degree sequence. Here each
conjunct is a cardinality constraint on a row of $A$.
\begin{eqnarray}
  \varphi^{\tuple{d_1,\ldots,d_n}}_{degSeq}(A) &=& 
         \bigwedge_{1\leq i\leq n} \left(\sum_{j=1}^n A_{i,j} = d_i\right)
\end{eqnarray}

\subsection*{Computing Highly Irregular Graphs Per Degree Sequence}
 
A connected graph is called highly irregular if each of its vertices
is adjacent only to vertices with distinct
degrees~\cite{alavi1987highly}. The number of highly irregular graphs
with $n\leq 15$ vertices is detailed as sequence \texttt{A217246} in
the OEIS~\cite{oeis}. By application of instance dependent canonizing
symmetry breaks we extend this sequence with 4 new values.
The following constraint specifies that the graph represented by
adjacency matrix $A$ with degree sequence $\tuple{d_1,\ldots,d_n}$ is
highly irregular.
\begin{equation}
  \varphi_{hi}^{\tuple{d_1,\ldots,d_n}}(A) =
  \hspace{-5mm}\bigwedge_{1\leq i,j<k\leq n ~s.t~  d_j = d_k}\hspace{-4mm}
                 (\neg A_{i,j} \vee \neg A_{i,k} )
  \wedge              
  \varphi^{\tuple{d_1,\ldots,d_n}}_{degSeq}(A)
  \wedge               
  \varphi_{con}^{n}(A)
     \label{constraint:hig}
\end{equation}  
Here, the formula $\varphi_{con}^{n}(A)$ specifies that the graph
represented by adjacency matrix $A$ is connected.  The following
constraint introduces propositional variables $p^k_{i,j}$ to express
that vertices $i$ and $j$ are connected by a path that consists of
intermediate vertices from the set $\{1,\ldots,k\}$.
\begin{eqnarray}
    \varphi_{con}^{n}(A) &=&
        \hspace{-2mm} \bigwedge_{1\leq i,j\leq n} \hspace{-2mm}
                   (p_{i,j}^{0} \leftrightarrow A_{i,j})  ~\land~ 
  \hspace{-4mm} \bigwedge_{1\leq i,j,k\leq n} \hspace{-4mm}
  p_{i,j}^{k} \leftrightarrow (p_{i,j}^{k-1} \vee (p_{i,k}^{k-1} \land
  p_{k,j}^{k-1})) ~\wedge 
\nonumber
\\
  & & ~\land~ \hspace{-3mm} \bigwedge_{1\leq i,j\leq n} \hspace{-2mm}(p_{i,j}^{n})
  \label{constraint:connected}
\end{eqnarray}

Our strategy is to compute all highly irregular graphs with $n$
vertices in three steps: (1) We compute the set of degree sequences
for all highly irregular graphs with $n$ vertices; (2) For each degree
sequence we compute an instance dependent canonizing symmetry break;
(3) We apply per degree sequence, the instance dependent canonizing
symmetry break to compute the corresponding set of graphs with the
corresponding degree sequence.

To perform the first step we apply a result from
\cite{majcher1997degree} which states that any degree sequence of a
highly irregular graph is of the form
$\tuple{m^{n_m},\ldots,i^{n_i},\ldots,1^{n_1}}$  where:
\begin{enumerate}
\item 
 $n_i \geq n_m$ for $1\leq i \leq m$; and
 \item 
 $\sum_{i=1}^{m} (n_{i}*i)$ and $n_m$ are positive even numbers.
\end{enumerate}
It is straightforward to enumerate all degree sequences for graphs with
up to 20 vertices that satisfy this property.  We then apply a SAT
solver to determine which of these sequences is the degree sequence
of some highly irregular graph. 
Step (2) is performed using a SAT solver, per degree sequence, by
application of the above described adaptation of
Algorithm~\ref{CanSetAlgorithm} to compute an irredundant instance
dependent canonizing set with respect to
$\varphi_{hi}^{\tuple{d_1,\ldots,d_n}}(A)$.
In step (3) we enumerate all non-isomorphic highly irregular graphs
per degree sequence with respect to the corresponding canonizing
symmetry breaking constraints. We compute the graphs with a simple
backtrack based (exhaustive search) program written in Java in which
the variables of the adjacency matrix are assigned one by one until a
solution is found.

Table \ref{compute-hig} presents our results. The columns are divided
into three pairs corresponding to the three steps described above: the
first pair -- computing degree sequences, the second pair --
computing (instance dependent) canonizing permutation sets, and the
third pair -- computing solutions (using the derived canonizing
symmetry breaks). Each pair presents the computation size and
information on the solutions. For the first pair, the number of degree
sequences. For the second pair, the average number of permutations in
the canonizing permutation sets. In the third pair, the number of
connected highly irregular graphs with $n$ vertices (set in
boldface). The values for $n\leq 15$ vertices correspond to those
detailed as sequence \texttt{A217246} in the OEIS~\cite{oeis}. The
values for $n>15$ are new.

When computing solutions, as detailed in the rightmost pair of columns
of Table~\ref{compute-hig}, computation is performed in parallel,
using a separate thread of the cluster for each degree sequence found in
the first step. So for example, when $n=20$, there are 151 parallel
threads running with a total time of 7190.23 hours. This implies an
average of about 47 hours.
%
Note that we succeed to compute canonizing symmetry breaks for more
than 20 nodes. We have not included the results in
Table~\ref{compute-hig} as the subsequent graph enumeration problems
involve a humongous number of graphs.

\begin{table}\centering
  \caption{Computing highly irregular graphs per degree sequence (time in 
            seconds unless otherwise indicated).}
\label{compute-hig}
\begin{tabular}{|r|r|r|r|r|r|>{\bfseries}r|}\hline
     & \multicolumn{2}{c|}{deg.seqs}  
     & \multicolumn{2}{c|}{can. sets} 
     & \multicolumn{2}{c|}{solutions} \\ 
\hline 
      $n$ & \multicolumn{1}{c|}{time} &  deg.seqs     
      & \multicolumn{1}{c|}{time} &  perms (avg.) 
      & \multicolumn{1}{c|}{time}      &  sols (total) \\ 
\hline
11  &    1.11           &  2      &  3.28     &     6.5    &    0.29 & 21                 \\
12  &    4.47           &  7      &  15.16    &     7.57   &    0.87 & 110                \\
13  &    5.73           &  7      &  28.65    &    9.71    &    1.42 & 474               \\ 
14  &    17.85          & 16      &  93.69    &    10.56   &    5.62 & 2\,545              \\ 
15  &    27.15          & 17      &  183.49   &    13.11   &   28.39 & 18\,696          \\ 
16  &    59.69          & 33      &  487.85   &    13.57   &  234.97 & 136\,749         \\ 
17  &    111.97         & 38      &  683.14   &    13.94   & 3\,312.04 & 1\,447\,003     \\ 
18  &    237.53         & 68      &  1\,797.16  &    14.89 &  14.17 hr.& 18\,772\,435    \\ 
19  &    468.99         & 92      &  3\,281.07  &    16.07 & 263.90 hr.  & 303\,050\,079   \\ 
20  &    881.53         & 151     &  8\,450.91    &    16.73 &  7190.23 hr. &  6\,239\,596\,472\\ 

 \hline
\end{tabular}
\end{table}

\section{A Generalization to Matrix Models}
\label{sec:matrixmodels}

Graph search problems, as considered in this paper, are a special case
of matrix search problems expressed in terms of a matrix of finite
domain decision variables
\cite{flener2002,frisch2003,yip2011,Walsh12}.  Often, in such
problems, both rows and columns can be permuted, possibly by different
permutations, while preserving solutions. Matrix search problems with
this property are called
``\emph{fully-interchangeable}''~\cite{yip2011}.  A graph search
problem can be seen as a fully-interchangeable matrix search problem
where the variables are Boolean, the matrix is symmetric and has
$\false$ values on its diagonal, and symmetries involve permuting rows
and columns, but only by the same permutation for both.

Similar to Definition~\ref{def:order}, it is common to define a
lex-order on matrix models. For matrices $M_1$ and $M_2$ (of the same
dimension) with $s_1$ and $s_2$ the strings obtained by concatenating
their rows, $M_1\preceq M_2$ if and only if $s_1\preceq_{lex}
s_2$. Similar to Definition~\ref{def:canonizing}, the \textsc{LexLeader}
method~\cite{Crawford96} can be applied to a fully-interchangeable
matrix search problem to guarantee canonical solutions which are
minimal in the lex ordering of matrices. For an $n\times m$ matrix
search problem this involves potentially considering all $n!\times m!$
pairs of permutations (per $n$ rows and per $m$ columns). This is not
practical as it means introducing $n!\times m!$ lex constraints on
strings of size $n\times m$.

To this end, the \textsc{DoubleLex} constraint was introduced in
\cite{flener2002} to lexicographically order (linearly) both rows and
columns.  The \textsc{DoubleLex} constraint can be viewed as derived
by a subset of  the constraints imposed in the \textsc{LexLeader}
method~\cite{Walsh12}. For a matrix with $n$ rows and $m$ columns this
boils down to a total of only $(n-1)+(m-1)$  permutations.
The \textsc{DoubleLex} constraint has been shown to be very effective
at eliminating much of the symmetry in a range of
fully-interchangeable matrix search problem. Still, it does not break
all of the symmetries broken by \textsc{LexLeader}.

We compare the \textsc{DoubleLex} symmetry break with canonical
symmetry breaking for the  application to {EFPA} (Equi-distant
Frequency Permutation Array).
An instance of the EFPA problem takes the form $(q,\lambda,d,v)$. The
goal is to find a set of $v$ words of length $q\lambda$ such that each
word contains $\lambda$ copies of the symbols 1 to $q$, and each pair
of words is Hamming distance $d$ apart. The problem is naturally
expressed as a $v\times q\lambda$ (fully-interchangeable) matrix
search problem.

Table~1 in the survey by Walsh \cite{Walsh12} illustrates the
power of the \textsc{DoubleLex} symmetry break. The table details the
number of solutions found with \textsc{DoubleLex} constraint for
several instances of the EFPA problem in contrast to the total number
of non-symmetric solutions. It demonstrates that \textsc{DoubleLex}
leaves very few redundant solutions.

We have adapted Algorithms~\ref{CanSetAlgorithm}
and~\ref{ReduceAlgorithm} so that they apply to search for pairs of
permutations which induce constraints to break all symmetries in
fully-interchangeable matrix search problems. With these constraints
we obtain only the canonical solutions. We have applied such
constraints to the instances of the EFPA problem considered
in~\cite{Walsh12}.
For matrix search problems we initialize
Algorithm~\ref{CanSetAlgorithm} taking $\Pi$ to include the
permutation pairs corresponding to the \textsc{DoubleLex} symmetry break.

Table~\ref{EFPA}  summarizes our results obtained, as
all results in this paper, using the finite-domain constraint compiler
\bee~\cite{jair2013} with the underlying Glucose 4.0 SAT
solver~\cite{Glucose}.
On the left side the table details statistics for solutions with
\textsc{DoubleLex}: the number of permutation pairs introduced by
\textsc{DoubleLex}, the solving time (in seconds) and the number of
solutions found. The right side of the table details the same for the
canonizing symmetry breaks. Except that here the number of permutation
pairs are for the canonizing symmetry breaks, as discovered using
Algorithms~\ref{CanSetAlgorithm} and~\ref{ReduceAlgorithm}. Here we
also make explicit the number of additional permutations $\Delta$, in
addition to those introduced by \textsc{DoubleLex}, required to
provide a canonizing symmetry break. 
In several rows of the table, corresponding to instances where
\textsc{DoubleLex} is in fact complete, this value is negative.  In
these cases no permutations were added by
Algorithm~\ref{CanSetAlgorithm} and several were removed by
Algorithm~\ref{ReduceAlgorithm} when deriving the corresponding
canonizing symmetry break. It is interesting to note that, often
times, for canonical symmetry breaks, only a few permutations need be
added on top of these already introduced by \textsc{DoubleLex}.
The numbers in the rightmost columns (set in boldface) correspond to
the number of distinct non-symmetric solutions. 
The corresponding sets of permutation pairs for
the instances in Table~\ref{EFPA} can be obtained from
\url{http://www.cs.bgu.ac.il/~mcodish/Papers/Tools/canonizingSets}.

\begin{table}
\centering
\caption{Number of solutions for EFPA with 
    \textsc{DoubleLex} and canonizing symmetry breaks.}
\label{EFPA}
\begin{tabular}{|r|r|r|r|rr|r|>{\bfseries}r|}
\hline
 & 
\multicolumn{3}{c|}{\textsc{DoubleLex} sym. break}  
                                         &
                                         \multicolumn{4}{c|}{canonizing
                                         sym. break} \\ \hline
              $(q,\lambda,d,v)$  & \multicolumn{1}{r|}{perms} 
              & \multicolumn{1}{r|}{sat (sec.)}
              & \multicolumn{1}{r|}{sols} 
              & \multicolumn{2}{r|}{perms ($\Delta$)} 
              & \multicolumn{1}{r|}{sat (sec.)}
              & \multicolumn{1}{r|}{sols} \\ \hline    
$(3,3,2,3)$      & 10   & 0.01       &  6         &  8   & (-2)  & 0.01   &  6    \\ 
$(4,3,3,3)$      & 13   &  0.06      &  16        &  16  & (3)   & 0.06   & 8     \\
$(4,4,2,3)$      & 17   &  0.05      &   12       &  15  & (-2)  & 0.03   &  12   \\ 
$(3,4,6,4)$      & 14   &  19.12     &  11\,215   &  328 & (314) & 14.10  & 1\,427  \\
$(4,3,5,4)$      & 14   &  145.22    &  61\,258   &  1537& (1523)& 414.56 & 8\,600  \\ 
$(4,4,5,4)$      & 18   &  280.33    &  72\,251   &  1793& (1775)& 748.77 & 9\,696  \\ 
$(5,3,3,4)$      & 17   &  0.29      &  20        &  27  & (10)  & 0.22   & 5     \\ 
$(3,3,4,5)$      & 12   &  0.36      &  71        &  36  & (24)  & 0.45   & 18    \\
$(3,4,6,5)$      & 15   &  611.88    &  77\,535   &  988 & (973) & 195.17 & 4\,978  \\ 
$(4,3,4,5)$      & 15   &  11.34     &  2\,694    &  245 & (230) & 9.51   & 441   \\ 
$(4,4,2,5)$      & 19   &  0.10      &  12        &  15  & (-4)  & 0.11   &  12   \\ 
$(4,4,4,5)$      & 19   &  22.42     &  4\,604    &  385 & (366) & 25.41  & 717   \\
$(4,6,4,5)$      & 27   &  46.88     &  5\,048    &  441 & (414) & 75.81  & 819   \\
$(5,3,4,5)$      & 18   &  157.94    &  20\,831   &  898 & (880) & 262.02 & 3\,067  \\
$(6,3,4,5)$      & 21   &  1230.19   &  111\,082  &  2348& (2327)& 3537.54& 15\,192     \\ \hline
\end{tabular}
\end{table}
%

                                          
               

\section{Related Work}

Isomorphism free generation of combinatorial objects and particularly
graphs, is a well studied topic
\cite{Read78,Faradzev1978,read1981survey,brinkmann1996fast,mckay1998isomorph}. Methods
that generate the canonical representatives of each equivalence class
are sometimes classified as ``orderly'' generation methods. This is a
dynamic approach. Typically graphs are constructed by adding edges in
iteration until a solution is found and backtracking when failing. In
each such iteration the graph is checked to determine whether it can
still be further extended: (a) to a solution of the graph search
problem, and (b) to a canonical graph. Both of these tests consider
only the fixed part of the partial graph.  These techniques do not
restrict the set of permutations to be canonical but rather consider
all permutations relevant to the partially instantiated
structure. Still, initially, there are very few permutations that need
to be considered for (b) as the partial graph is still small. However,
as the partial graph becomes more instantiated this test becomes
harder and consumes more time. This approach is also the one applied
in \cite{yip2011} for matrix search problems.
In contrast our method is static. We aim to compute, before applying
search, a small set of permutations that apply to break the symmetries
in solutions. Our approach does not rely on which parts of the graph
have already been determined during search.

\section{Conclusion and Future Work}
\label{sec:conclusion}

We have illustrated the applicability of canonizing symmetry breaking
constraints for small graph and matrix search problems. Although any
row/column permutation is potentially a symmetry, we compute compact
canonizing symmetry breaks, much smaller than those which consider all
permutations.  Our strategy is two phase. First, symmetry breaking
constraints are computed. Second, these constraints are added to the
model and then any solver can be applied to find (all) solutions which
satisfy the model.

For graph search problems, we have presented methods that generate
both instance independent and instance dependent symmetry breaking
constraints.  While instance dependent symmetry breaks have limited
applicability since they grow enormously for graphs with more than 10
vertices, instance dependent symmetry breaks have been successfully
applied to compute new values in highly irregular graphs OEIS sequence
for graphs with up to 20 vertices. For matrix search problems our
focus is on instance dependent constraints. 

Although, our approach is applicable only to graphs with
small numbers of vertices, there are many open small graph search
problems. For example the set of all Ramsey $R(4,5;24)$ graphs has not
been determined yet. We are currently trying to extend our techniques
to apply to compute symmetry breaks for this problem which involves
only 24 vertices.

Finally, we note that our approach can also apply to improve dynamic
symmetry breaking techniques. Given a partially instantiated graph, to
determine if it is extendable to a canonical graph, one need not
consider all of the permutations related to the already instantiated
part. This is because some of those permutations are redundant.

\subsection*{Acknowledgments}
We thank the anonymous reviewers of an earlier version of this paper
for their constructive suggestions. In particular the addition of
Section~\ref{sec:matrixmodels} is in view of the comments of the
reviewers.

\bibliographystyle{spmpsci}      

\end{document}